\documentclass[conference]{IEEEtran}
\IEEEoverridecommandlockouts

\usepackage{amssymb}
\usepackage{amsmath}
\usepackage{amsthm}

\usepackage{graphicx}
\usepackage{xspace}
\usepackage{subcaption}
\usepackage{algorithmic}
\usepackage{algorithm}
\usepackage{color}
\usepackage{hyperref}
\usepackage{textcomp}

\newcommand{\x}{\mathcal X}
\newcommand{\xfree}{\mathcal X_{\text{free}}}
\newcommand{\xobs}{\mathcal X_{\text{obs}}}
\newcommand{\xgoal}{\mathcal X_{\text{goal}}}
\newcommand{\xinit}{x_{\mathrm{init}}}
\newcommand{\N}{\mathbb{N}}
\newcommand{\iterm}{i_\mathrm{term}}

\newtheorem{theorem}{Theorem}
\newtheorem{problem}{Problem}

\newcommand{\fmt}{FMT$^*$\xspace}
\newcommand{\gmt}{GMT$^*$\xspace}

\newcommand{\prm}{PRM$^*$\xspace}

\usepackage{multicol}
\usepackage{color}

\newcommand{\G}{\mathcal{G}}   
\newcommand{\dr}{\lambda}   

\newcommand{\vunexplored}{V_{\mathrm{unexplored}}}
\newcommand{\vopen}{V_{\mathrm{open}}}
\newcommand{\vclosed}{V_{\mathrm{closed}}}

\DefineNamedColor{named}{orange}{rgb}{1,0.6,0}

\newcommand\blfootnote[1]{%
  \begingroup
  \renewcommand\thefootnote{}\footnote{#1}%
  \addtocounter{footnote}{-1}%
  \endgroup
}

\renewcommand{\baselinestretch}{0.955}

\begin{document}

\title{Group Marching Tree: Sampling-Based \\
Approximately Optimal Motion Planning on GPUs
}

\author{\IEEEauthorblockN{Brian Ichter}
\IEEEauthorblockA{Aeronautics \& Astronautics\\
Stanford University\\
Stanford, California 94305\\
Email: ichter@stanford.edu}
\and
\IEEEauthorblockN{Edward Schmerling}
\IEEEauthorblockA{Institute for Computational \&\\
Mathematical Engineering\\
Stanford University\\
Stanford, California 94305\\
Email: schmrlng@stanford.edu}
\and
\IEEEauthorblockN{Marco Pavone}
\IEEEauthorblockA{Aeronautics \& Astronautics\\
Stanford University\\
Stanford, California 94305\\
Email: pavone@stanford.edu}}

\maketitle

\begin{abstract}
This paper presents a novel approach, named the Group Marching Tree (\gmt) algorithm, to planning on GPUs at rates amenable to application within control loops, allowing planning in real-world settings via repeated computation of near-optimal plans.
\gmt, like the Fast Marching Tree (\fmt) algorithm, explores the state space with a ``lazy'' dynamic programming recursion on a set of samples to grow a tree of near-optimal paths.
\gmt, however, alters the approach of \fmt with approximate dynamic programming by expanding, in parallel, the group of all active samples with cost below an increasing threshold, rather than only the minimum cost sample.
This group approximation enables low-level parallelism over the sample set and removes the need for sequential data structures, while the ``lazy'' collision checking limits thread divergence---all contributing to a very efficient GPU implementation.
While this approach incurs some suboptimality, we prove that \gmt remains asymptotically optimal up to a constant multiplicative factor.
We show solutions for complex planning problems under differential constraints can be found in \texttildelow10~ms on a desktop GPU and \texttildelow30~ms on an embedded GPU, representing a significant speed up over the state of the art, with only small losses in performance.
Finally, we present a scenario demonstrating the efficacy of planning within the control loop (\texttildelow100~Hz) towards operating in dynamic, uncertain settings.
\blfootnote{This work was supported by a Qualcomm Innovation Fellowship and by NASA under the Space Technology Research Grants Program, Grant NNX12AQ43G. Brian Ichter was supported by the DoD NDSEG Program.}
\end{abstract} 


\IEEEpeerreviewmaketitle

\section{Introduction}\label{sec:intro}

Robotic systems are increasingly operating in real-world settings---away from the structure, repetition, and certainty of the factory floor---that require a robot to not only sense its environment and state in real time, but to react accordingly \cite{Lavalle2011}. 
Acting in these paradigms often necessitates motion plans be computed on the basis of limited state and environmental knowledge, both of which may vary rapidly as information is gathered and the robot's surroundings change.
A major challenge in this approach is thus replanning quickly, ideally up to the bound of the control feedback loop frequency (\texttildelow100~Hz), particularly for systems governed by dynamic constraints operating in complex environments.

Sampling-based motion planning has emerged as an especially successful paradigm for rapid planning in complex, high-dimensional, and unstructured environments \cite{Lavalle2006}, and it has been shown to extend well to planning with differential constraints \cite{SchmerlingJansonEtAl2015}\cite{SchmerlingJansonEtAl2015b}.
These methods
probe the state space with a set of samples to be connected, under the supervision of a collision detection module, to form a traversable graph representation of the free state space.
Sampling-based roadmap methods, such as the probabilistic roadmap algorithm (PRM) \cite{KavrakiSvestkaEtAl1996} and its asymptotically optimal variant \prm \cite{KaramanFrazzoli2011}, initially construct a graph where samples are connected to each of their near neighbors, provided the connection is collision-free.
A shortest path search is performed on the resulting roadmap to yield solutions (up to the resolution constraints of the underlying graph) to the optimal planning problem \cite{Lavalle2006}.
As these methods are limited in their speed by the initial graph building stage, variants have been developed that simultaneously construct the graph edges while searching, e.g., Lazy PRM \cite{BohlinKavraki2000}, which avoid performing any collision checks that are not required during the roadmap shortest path computation.
The Fast Marching Tree algorithm (\fmt) further reduces collision checking by implementing direct dynamic programming (as opposed to full shortest path search) while constructing a tree subgraph of a disk graph defined by connection cost, increasing performance particularly in complex, high-dimensional spaces \cite{JansonSchmerlingEtAl2015}.
Yet even with these advances, path plan computation times are often over an order of magnitude greater than the periods of controller loops and with the slowing growth rate of CPU computational power (due primarily to limited clock frequency), it is unlikely raw CPU power will soon bridge this gap.
We instead propose algorithm development for a different paradigm: parallel computing, with a particular focus on development for the interplay between algorithmic design and the many thousand core architectures of GPUs.
Unfortunately, while we seek a solution inspired by the dynamic programming literature for a single pair of start/goal states (the use case relevant to control loop planning), the inherently sequential nature of dynamic programming's minimum cost node expansion, which, e.g., \fmt is built on, complicates the necessary massive parallelization.

\emph{Statement of Contributions.}
In this work, we propose the use of approximate dynamic programming (ADP) methods that leverage algorithm parallelism for greater speed while incurring only a bounded degree of suboptimality.
We present the Group Marching Tree (\gmt) algorithm that, like the Fast Marching Tree algorithm (\fmt) \cite{JansonSchmerlingEtAl2015}, performs a ``lazy'' dynamic programming recursion on a set of samples in the state space to grow a tree of near-optimal cost-to-arrive paths. 
\gmt, however, varies from the approach of \fmt with ADP by expanding the tree, in parallel, from the group of all active samples with cost below a threshold, rather than only the minimum cost sample (essentially locally relaxing the principle of optimality). 
This group approximation enables low-level parallelism over the sample set and removes the need for sequential data structures, allowing for massive parallelization on GPUs.
The ``lazy'' collision checking further facilitates GPU implementation by limiting thread divergence at the lowest levels.
While these approximations do introduce some suboptimality, we prove that \gmt remains asymptotically optimal up to a constant multiplicative factor and demonstrate through numerical experiments that the empirical loss is well below the theoretical bound. 

We further discuss the implementation of \gmt on GPU architectures and show its application to several illustrative motion planning problems with differential constraints, for which we consider kinodynamic and nonholonomic planning.
These numerical experiments show that solution trajectories can be computed in \texttildelow10~ms on a consumer grade GPU and \texttildelow30~ms on an embeddable GPU; achieving computation times two orders of magnitude faster than a state-of-the-art CPU algorithm and an order of magnitude faster than a state-of-the-art GPU algorithm, again with only small performance losses. 
Lastly, we demonstrate the efficacy of planning within the control loop on a simplified quadrotor in a collapsing cave environment, with state disturbances and environmental dynamism. 

\emph{Related Work.} 
Previous works have addressed planning in real-world settings through a number of methods.
One approach is that of feedback motion planning, which traditionally defines a policy over the state space to allow the current state to be fed back into the controller \cite{Lavalle2011}. 
Unfortunately, the ephemeral nature of the planning environment complicates this process; while ideally these feedback plans would always reflect the current knowledge state, they may become quickly outdated and inaccurate if updating or recomputing plans is too computationally intensive.
Some methods exist to simplify this computation, such as \cite{LindemannLaValle2009}, which computes a field of guiding vectors over the entire free state space, however they generally must interpolate over the state space to define the local action and assume the state space can be easily represented.
The high-frequency replanning approach of \cite{SunPatilEtAl2015} uses a similar approach to our own by quickly replanning full trajectories. 
This work leverages parallelism to generate many rapidly-exploring random tree (RRT) trajectories, selecting the trajectory with the lowest collision probability at each computation step.
Our work focuses on construction of a single tree to allow planning within the loop at rates of \texttildelow100~Hz, rather than the 4~Hz considered in \cite{SunPatilEtAl2015}. 
Other works, e.g., RRT$^\text{X}$ \cite{OtteFrazzoli2016}, have accelerated the replanning approach by iteratively rewiring a single tree as new information becomes available. 
By reusing the previous tree at each time step, these methods are limited in scenarios where the environment changes drastically. 
Furthermore, the RRT$^\text{X}$ tree is rooted at the goal state to enable use in scenarios with disturbances, but this limits its utility in problems with a changing goal state.

Another approach to planning in changing environments is to couple low-frequency global planners with high-frequency reactive controllers that determine actions which are collision-free and optimal in a local sense. 
Using methods such as precomputed trajectory libraries and funnels \cite{MajumdarTedrake2012}, learning \cite{RichterVega-BrownEtAl2015}, and potential fields \cite{SchererSinghEtAl2008}, this approach has shown practical success in many settings, however, its focus on the local region can ignore variations in global reachability resulting from actions (e.g., not accounting for momentum or nonholonomic constraints) and their use of heuristics to inform actions may incur suboptimality (e.g., in maze-like environments).
In this work \gmt is shown to be capable of computing motion plans in a tempo comparable to reactive controllers, lessening the need for consideration of only local actions.
\gmt can further be used in concert with reactive controllers to better inform actions via accurate cost-to-go computations, thus potentially benefiting from properties like the robustness in \cite{MajumdarTedrake2012}.

A main tenet of our work is the use of approximate dynamic programming (ADP) to allow parallelism while exploring the state space. 
Similar search methodologies, i.e., expanding wavefronts in low-cost groups, have been used successfully for graph search to allow parallelism and complexity reduction.
Dial's algorithm \cite{Dial1969} implements Dijkstra's algorithm on graphs with integer weights by stepping through buckets, exactly solving for the shortest paths.
The $\Delta$-stepping algorithm \cite{MeyerSanders2003} generalizes Dial's algorithm to solve exactly the single source shortest path problem on non-negative real-value weighted graphs by successively relaxing edges while stepping through buckets with width $\Delta$, however it performs extra work by revisiting edges to maintain exactness.
The Group Marching Method \cite{Kim2001} builds on the Fast Marching Method (an inspiration for \fmt) to solve the eikonal equations by advancing a group of points together in two iterations, the first forward and the second backwards to correct for instabilities. 
Our work employs a similar expansion strategy, but abandons any additional computation necessary to maintain exactness, instead leveraging the underlying disk graph to maintain asymptotic optimality within a constant factor.

Parallelization too has been applied successfully to motion planning by a number of researchers, finding significant algorithm accelerations.
An early result in sampling-based motion planning showed that probabilistic roadmap methods are embarrassingly parallel \cite{AmatoDale1999}, which was later extended to implementation on GPUs \cite{PanLauterbachEtAl2010}.
The focus of PRM-based approaches on entire graph construction however can be prohibitively slow even with GPUs.
Common approaches to parallelization of sampling-based planning include focusing only on algorithm subroutines (such as collision checking and nearest neighbor search) \cite{PanLauterbachEtAl2010}\cite{BialkowskiKaramanEtAl2011}, 
adapting serial algorithms via AND/OR-parallelism \cite{SunPatilEtAl2015}\cite{DevaursSimeonEtAl2013},
or using load balancing and domain decomposition \cite{RodriguezDennyEtAl2013}\cite{FidelJacobsEtAl2014}.
\gmt's ADP is parallel at the sample level, meaning many of these methodologies (e.g., collision checking, domain decomposition) are applicable in implementation. 
Furthermore, this low-level parallelism enables massive parallelization for use on GPUs.

\emph{Organization.} The remainder of this document is organized as follows. 
Section~\ref{sec:setup} describes the problem setup. 
Section~\ref{sec:alg} discusses the \gmt algorithm and proves its asymptotic optimality up to a constant multiplicative factor. 
Section~\ref{sec:gpu} discusses its implementation on GPUs. 
Section~\ref{sec:mp} demonstrates the performance of \gmt with motion planning problems under differential constraints.
Lastly, Section~\ref{sec:concl} summarizes our findings and proposes directions for future work.

\section{Problem Setup}\label{sec:setup}

In this section and the next (which contains the description and analysis of the \gmt algorithm), for ease of exposition we consider the geometric planning problem---the problem, loosely speaking, of computing the shortest free path from an initial state to a goal region where any two states can be connected by a straight line. 
The problem is briefly overviewed here, but a full, detailed problem formulation can be found in \cite{JansonSchmerlingEtAl2015}.
Let $\x = [0,1]^d$ be the state space, where $d\in \mathbb{N},d\geq 2$. 
Let $\xobs$ be the obstacle space, $\xfree = \x \setminus \xobs$ be the free state space, $\xinit\in\xfree$ be the initial condition, and $\xgoal \subset \xfree$ be the goal region.
A path is said to be \emph{collision-free} if $\sigma(\tau) \in \xfree$ for all $\tau \in [0,1]$.
A path is said to be \emph{feasible} if it is collision-free, $\sigma(0) = \xinit$, and $\sigma(1) \in \xgoal$. 

\vspace{0.15cm}
\begin{problem}[Optimal path planning]
Given a path planning problem $(\xfree, \xinit, \xgoal)$ and a cost measure $c$, find a feasible path $\sigma^\ast$ such that $c(\sigma^\ast) = \min \{c(\sigma):\sigma$ is feasible$\}$. If no such path exists, report failure.
\end{problem}
\vspace{0.15cm}

For Section~\ref{sec:alg} we consider the cost measure $c(\sigma)$ as the arc length of $\sigma$ with respect to the Euclidean metric and write $\|y-x\|$ to denote the cost of the shortest path between $x, y \in \x$.\footnote{To accommodate alternative dynamics/costs we may instead consider $d(x,y)$, the result of solving an optimal two-point boundary value problem connecting $x$ to $y$, and replace any discussion of connection balls in Section~\ref{sec:alg} with the notion of bounded-cost reachable sets.} Though this is arguably the simplest formulation of robotic motion planning, we note the analytical distinction between (a) determining, for more general problem setups with differential constraints or alternative costs, whether a sample set in $\x$ admits a near-optimal trajectory as a sequence of local connections, and (b) arguing that a planning algorithm is capable of identifying such a high-quality solution given a sample set.
We refer the reader to our previous works \cite{SchmerlingJansonEtAl2015}\cite{SchmerlingJansonEtAl2015b}\cite{JansonIchterEtAl2015} for discussion on the first point, including expressions for local connection radii under both randomized and deterministic state space sampling, and abbreviate the relevant discussion (Theorems~\ref{thm:GMTpe} and~\ref{thm:GMTao}) in the present work. 
The novel theoretical contribution of this paper is establishing that \gmt, which achieves a high degree of parallelism in a single query approach unlike the \fmt and \prm algorithms analyzed in those works, still recovers asymptotically optimal paths (up to a constant factor) under the same sampling and connection radius assumptions (Theorem~\ref{thm:GMTbound}).
Our numerical experiments in Section~\ref{sec:mp} consider both kinodynamic and nonholonomic planning problems, specifically double integrator and Dubins airplane dynamics.

\section{The Group Marching Tree Algorithm}\label{sec:alg}

\subsection{\texorpdfstring{\gmt}{Group Marching Tree}}\label{sec:gmt}

We now detail the Group Marching Tree algorithm (\gmt) to be used to approximately solve the optimal path planning problem.
\gmt performs a ``lazy'', approximate dynamic programming recursion to grow a tree of paths in cost-to-arrive space.
This amounts to iteratively attempting to expand all samples in the tree branches below a constantly increasing cost threshold, rather than expanding only the minimum cost sample.
The resulting algorithm enables simultaneous graph building and exploration of the state space, in a manner amenable to complex planning problems, such as high-dimensional, cluttered environments and differential constraints.

A description of the algorithm is given in Alg.~\ref{alg:GMT}, with a single iteration visualized in Fig.~\ref{fig:gmt_exp}. The algorithm takes as input the planning problem $(\xfree, \xinit, \xgoal)$, a sample set $\vunexplored$ of $n$ samples in $\xfree$ (at least one in $\xgoal$), a connection radius $r$, and a group cost threshold factor $\dr\in(0,1]$. 
Together, $\dr$ and $r$ define the group cost threshold increment $\delta$, equal to $\dr r$.
We refer to nodes (or interchangeably, samples) as neighbors if the connection cost between them is less than the connection radius $r$, as defined in Theorem \ref{thm:GMTpe}; 
$r = 4(1+\eta)^{1/d}\big(1/d\big)^{1/d} \big(\mu(\xfree)/\zeta_d\big)^{1/d} \big(\log n/n\big)^{1/d}$ where $\eta \geq 0$ is a tuning parameter, $\mu(\xfree)$ denotes the $d$-dimensional Lebesgue measure of $\xfree$, and $\zeta_d$ denotes the volume of the unit ball in $d$-dimensional Euclidean space. This $r$ applies for geometric planning with Euclidean cost and $\vunexplored$ sampled uniformly randomly from $\xfree$; smaller $r$ may be considered if $\vunexplored$ is sampled with low-dispersion, deterministic sequences \cite{JansonIchterEtAl2015}.

\begin{algorithm}
\caption{Group Marching Tree Algorithm}
\label{alg:GMT}
\begin{algorithmic}[1]
{\small 
\REQUIRE connection radius $r$, group cost threshold factor $\dr$, \\ set $\vunexplored$ of $n$ samples in $\xfree$, at least one in $\xgoal$
\STATE Place $\xinit$ in $\vopen$, set $i=0$ and $\delta = \dr r$ \label{line:gmtinit}
\STATE Initialize tree with root node $\xinit$ \label{line:treeinit}
\STATE Find nodes $\G$ in $\vopen$ with cost $\leq i\delta$ \label{line:group}
\STATE For each unexplored neighbor, $x$, of any node in $\G$: \label{line:parfor}
\STATE $\quad$ Find neighbor nodes $y$ in $\vopen$ \label{line:y}
\STATE $\quad$ Find locally-optimal connection to $x$ from a node in $y$ \label{line:localOpt}
\STATE $\quad$ If that connection is collision-free: \label{line:collision}
\STATE $\quad \quad$ Add edge to tree \label{line:treeadd}
\STATE $\quad \quad$ Remove $x$ from $\vunexplored$ and add to $\vopen$ \label{line:x}
\STATE Remove $\G$ from $\vopen$ and add to $\vclosed$ \label{line:removegroup}
\STATE Increment $i$ \label{line:increment}
\STATE Skip to Line \ref{line:group} until either: \label{line:term} \\ 
$\quad${\footnotesize a:} $\vopen$ is empty $\Rightarrow$ return failure\\ 
$\quad${\footnotesize b:} A node in $\G$ is in $\xgoal \Rightarrow$ return min cost path to $\xgoal$ 
}
\end{algorithmic}
\end{algorithm}

The algorithm proceeds by expanding a tree of paths outward through the state space, maintaining the samples in three sets: $\vunexplored$, $\vopen$, and $\vclosed$.
$\vunexplored$ consists of samples not yet added to the tree. 
$\vopen$ consists of samples added to the tree and still considered for expansion; intuitively these are the samples on the tree's outer ``branches'', i.e., close to the wavefront.
$\vclosed$ consists of samples added to the tree and no longer considered for expansion; intuitively these are the samples too far from the edge of the expanding tree to make any new connections.
The algorithm begins by adding only $\xinit$ to $\vopen$, setting $\vclosed$ empty, and initializing the tree of paths with $\xinit$ at its root (Lines \ref{line:gmtinit}-\ref{line:treeinit}).  
At each iteration $i$, all nodes in $\vopen$ with cost below $i \dr r$ ($i\delta$) are placed into a set $\G$, denoting the group of samples to be expanded in parallel (Line \ref{line:group}). 
The amount of parallelism of this step is controlled by a tuning parameter $\dr$, referred to as the group cost threshold factor, which represents a trade-off between parallelism and potential unconsidered optimal connections; $\dr \to 1$ represents expanding all nodes in the open set at once, resulting in nearly a breadth first search, while $\dr \to 0$ represents expanding only the minimum cost nodes in a given iteration, resulting in the same final solution as \fmt. 
The other side of the tradeoff, the unconsidered optimal connections, arises when multiple nodes along the optimal path are considered in the same group expansion, meaning they cannot connect to each other.
They may also occur when paths formed from the concatenation of several short connections fall behind the wavefront.
These effects, however, are curtailed by the $\dr$ factor, the underlying disk graph's structure, and the notion that longer steps will generally be more favorable than many short steps.
Even with this suboptimality, we show in Theorem \ref{thm:GMTao} that asymptotic optimality within a constant multiplicative factor is maintained; furthermore, the performance loss observed in practice is studied in Section~\ref{sec:numexp} and shown to be small. 
With the group of samples $\G$ in hand, all neighbors of $\G$ in $\vunexplored$ are considered for addition to the tree (Line \ref{line:parfor}). 
Denoting each unexplored neighbor of samples in $\G$ by $x$ and the neighbors of $x$ in $\vopen$ by $y$, \gmt then selects the locally-optimal connection, where locally-optimal is defined as the connection with the lowest cost for the previously computed path to $y$ concatenated with the straight line path from $y$ to $x$ (Line \ref{line:localOpt}).
Note that while this step potentially introduces suboptimal connections by \emph{lazily} ignoring the presence of obstacles, as in \fmt, these connections become vanishingly rare as the number of samples goes to infinity, as discussed and proven in \cite{JansonSchmerlingEtAl2015}.
If the selected connection is collision-free, it is added to the tree and $x$ is removed from $\vunexplored$ and added to $\vopen$ (Lines \ref{line:collision}-\ref{line:x}). 
When all samples have been considered, $i$ is incremented and the samples in $\G$ are removed from $\vopen$ and added to $\vclosed$ (Lines \ref{line:removegroup}-\ref{line:increment}).
The algorithm then moves to the next iteration, beginning at Line \ref{line:group}, or terminates if either $\vopen$ is empty or a node in $\G$ is in $\xgoal$ (Line \ref{line:term}).

\begin{figure}[htb!]
\centering
\begin{subfigure}[b]{0.23\textwidth}
    \includegraphics[width=\textwidth]{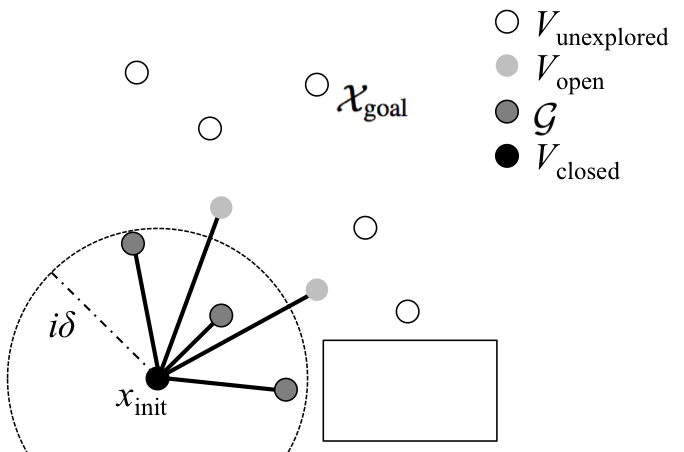}
    \caption{}\label{fig:gmt_exp1}
\end{subfigure}
\begin{subfigure}[b]{0.23\textwidth}
    \includegraphics[width=\textwidth]{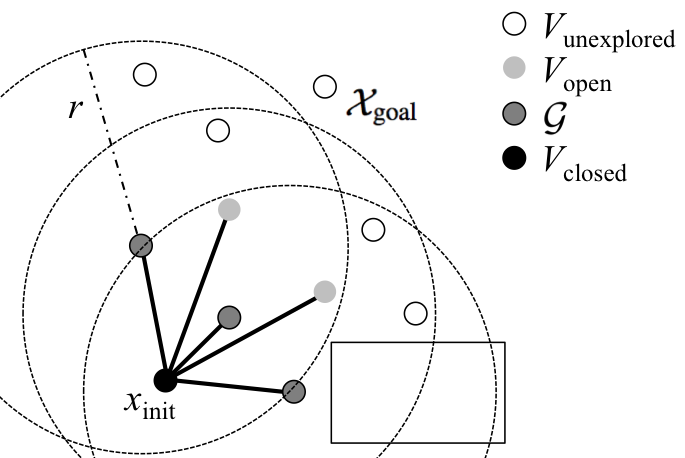}
    \caption{}\label{fig:gmt_exp2}
\end{subfigure}
\begin{subfigure}[b]{0.23\textwidth}
    \includegraphics[width=\textwidth]{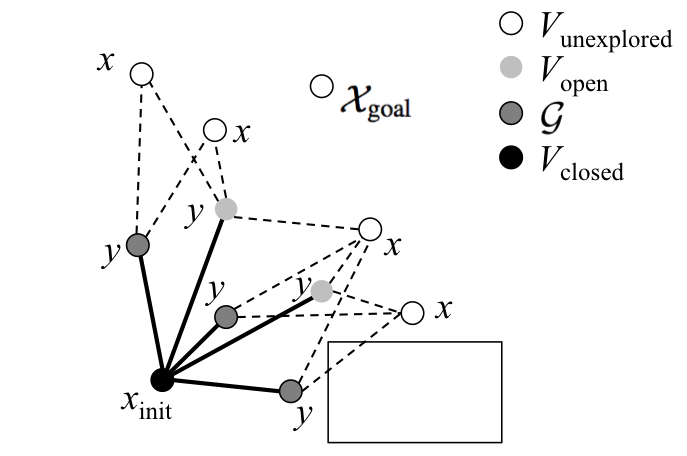}
    \caption{}\label{fig:gmt_exp3}
\end{subfigure}
\begin{subfigure}[b]{0.23\textwidth}
    \includegraphics[width=\textwidth]{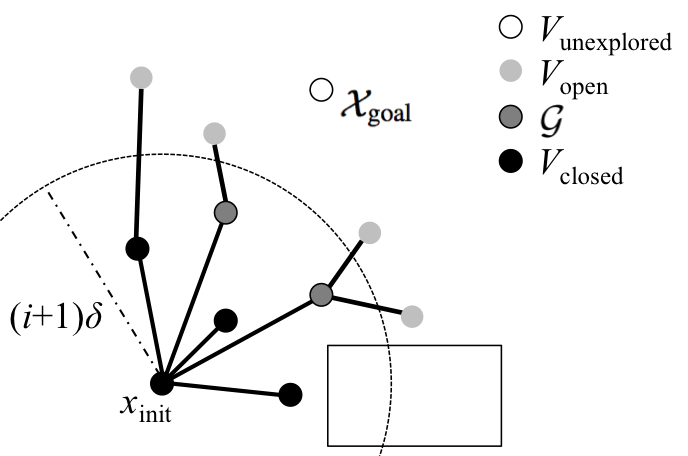}
    \caption{}\label{fig:gmt_exp4}
\end{subfigure}
\caption{One iteration step of \gmt expansion, labeled with (Fig., Line in Alg.~\ref{alg:GMT}). 
(\ref{fig:gmt_exp1}, Line \ref{line:group}) shows the selection of the group $\G$. 
(\ref{fig:gmt_exp2}, Line \ref{line:parfor}) shows the selection of the nearest neighbors of $\G$ in $\vunexplored$.
(\ref{fig:gmt_exp3}, Line \ref{line:y}) shows the candidate connections from which the locally optimal is chosen to connect the new samples.
(\ref{fig:gmt_exp4}, Line \ref{line:group}) shows the new tree after an iteration, Lines \ref{line:group}-\ref{line:increment}, and the new group $\G$.
Note, each group relies only on pathwise cost; the $i\delta$ ball is only shown for illustrative purposes.
}
\label{fig:gmt_exp}
\end{figure}

\subsection{\texorpdfstring{\gmt}{Group Marching Tree} Approximate Asymptotic Optimality}\label{sec:pe}

Our analysis of \gmt begins with the concept of \emph{probabilistic exhaustivity} as applied in related work establishing asymptotic optimality for a range of geometric \cite{StarekGomezEtAl2015} and differentially constrained \cite{SchmerlingJansonEtAl2015}\cite{SchmerlingJansonEtAl2015b} batch-processing, sampling-based motion planning algorithms. Briefly, probabilistic exhaustivity is the notion that within a sufficiently large set of uniformly sampled states, a sequence of samples approximating \emph{any} path arbitrarily well may be found. 
This property may be used to construct sample sequences approaching the optimal solution that are amenable for recovery by a planning algorithm.

In the subsequent analysis, we define $\mathtt{SampleFree}(n)$ to be a function that returns $n$ points sampled independently and identically from the uniform distribution on $\xfree$, at least one of which is in $\xgoal$.
We define a path $\sigma : [0,1] \to \x$ and a path $y : [0,1] \to \x$ that sequentially connects the sequence of waypoints $\{y_m\}^M_{m=0} \in \x$ with line segments. 
We say the sequence of waypoints $\{y_m\}$ $(\epsilon,r)\mathit{-traces}$ the path $\sigma$ if the following conditions hold: (i) $\lvert\lvert y_m - y_{m+1} \rvert\rvert \leq r$ for all $m$, (ii) the cost of $y$ is bounded as $c(y) \leq (1+\epsilon)c(\sigma)$, and (iii) the distance from any point $y$ to $\sigma$ is no more than $r$. 
We formally state this in Theorem \ref{thm:GMTpe} (proven as Theorem IV.5 in \cite{SchmerlingJansonEtAl2015}).

\begin{theorem}[Probabilistic Exhaustivity]\label{thm:GMTpe}
Define a planning problem $(\xfree, \xinit, \xgoal)$, a feasible path $\sigma : [0,1] \to \xfree$, a set of samples $V = \{\xinit\} \cup \mathtt{SampleFree}(n)$, and $\epsilon > 0$. 
For a fixed $n$ consider the event $\mathcal{A}_n$ that there exists $\{y_m\}^M_{m=0} \in V$, $y_0 = \xinit$, $y_M \in \xgoal$, which $(\epsilon,r)\mathit{-trace}$ $\sigma$, where
$$r = 4(1+\eta)^{\frac{1}{d}}\bigg(\frac{1}{d} \bigg)^{\frac{1}{d}} \bigg(\frac{\mu(\xfree)}{\zeta_d}\bigg)^{\frac{1}{d}}\bigg(\frac{\log n}{n}\bigg)^{\frac{1}{d}}$$
with $\eta \geq 0$. 
Then, as $n \to \infty$, the probability that $\mathcal{A}_n$ does not occur (denoted by its complement $\mathcal{A}^c_n$) is asymptotically bounded as P$[\mathcal{A}^c_n] = \text{O} (n^{-\frac{n}{d}}\log^{-\frac{1}{d}}n)$.
\end{theorem}

We now show that the cost of the path returned by \gmt is bound within a constant multiplicative factor of the cost of any such sequence of tracing waypoints.

\begin{theorem}[Bounded Suboptimality]\label{thm:GMTbound} Let $r > 0$ and suppose that the sequence of waypoints $\{y_m\}_{m=0}^M \subset \xfree$ satisfies $y_0 = \xinit$, $y_M \in \xgoal$, $\|y_m - y_{m-1}\| \leq r$ for all $m \in \{1,\dots,M\}$ and $B(y_m, r) \subset \xfree$ for all $m \in \{0,\dots,M\}$. Let $c_\text{\gmt}$ denote the cost of the path returned by \gmt using a connection radius $r$ and group cost threshold factor $\lambda$. Then
\[
c_\text{\gmt} \leq (1 + 2\lambda) \sum_{k=1}^{M} \|y_k - y_{k-1}\|.
\]
\end{theorem}
\begin{proof}

In the subsequent analysis we will refer to the open set of nodes at iteration $i$ as $\vopen(i)$.
We first assume, without loss of generality, that $\|y_m - y_{m-2}\| > r$ for all $m \in \{2,\dots,M\}$. This condition may be enforced by making a forward pass over the sequence and omitting the preceding point $y_{m-1}$ for any $y_m$ that violates the condition, without affecting the other lemma assumptions and only decreasing $\sum_{k=1}^{M} \|y_k - y_{k-1}\|$. Consider running \gmt to completion and for each $y_m$ let $c(y_m)$ denote the cost-to-arrive of $y_m$ in the generated tree. If $y_m$ is not contained in any edge of the tree, we set $c(y_m)=\infty$.
Let $i_m$ denote the first iteration after which $y_m$ has been added to the \gmt tree, that is,
\begin{align*}
i_m &= \min\{i \in \N \mid y_m \in \vopen(i)\}.
\end{align*}

Define $S_0 = 0$, $S_m = \sum_{k=1}^{m} \|y_k - y_{k-1}\|$, the cost of the path connecting $y_0, y_1, \dots, y_m$, and denote $\delta = \lambda r$. We show by induction that for all $m \in \{1,\dots,M\}$, one of the following two possibilities must hold:
\begin{equation}\label{eqn:DPclaim}
c_\text{\gmt} \leq S_m + m \delta,
\end{equation}
or
\begin{equation}\label{eqn:stepbound}
c(y_m) \leq S_m + (m-1) \delta \ \ \mathrm{and}\ \  i_m \leq \left\lceil \frac{S_{m-1}}{\delta}\right\rceil + m.
\end{equation}

The base case $m=1$ is trivial, since $\G_0$ contains only $\xinit = y_0$, and thus the first \gmt iteration makes every collision-free connection between $\xinit$ and the nodes contained in $B(\xinit, r)$, including $y_1$. Then $c(y_1) = \|y_1 - y_0\| = S_1$ and $i_1 = 1 = \lceil S_0 / \delta\rceil + 1$. Now suppose \eqref{eqn:DPclaim} or \eqref{eqn:stepbound} holds for $m-1$; that means that one of the following four statements must hold.
\begin{itemize}
\item[1.] $c_\text{\gmt} \le S_{m-1} + (m-1)\delta$,
\item[2.] $c(y_{m-1}) \leq S_{m-1} + (m-2)\delta$ and\\
$i_{m-1} \leq \left\lceil \frac{S_{m-2}}{\delta}\right\rceil + m - 1$ and
\begin{itemize}
\item[a.] \gmt ends before considering $y_m$ for connection ($i_m = \infty$), or
\item[b.] $y_{m-1} \in \vopen$ when $y_m$ is first considered for connection ($i_{m-1} \leq i_m - 1$), or
\item[c.] $y_{m-1} \notin \vopen$ when $y_m$ is first considered for connection ($i_{m-1} \geq i_m$).
\end{itemize}
\end{itemize}
We show that in each of these cases, \eqref{eqn:DPclaim} or \eqref{eqn:stepbound} holds for $m$.

\noindent Case 1: Since $\|y_m - y_{m-1}\| \geq 0$, we have
\begin{align*}
c_\text{\gmt} &\le S_{m-1} + (m-1)\delta \le S_m + m\delta.
\end{align*}

\noindent Case 2a:
The fact that $y_m$ goes unconsidered means that up until the time the algorithm terminates at iteration $\iterm$ (upon finding a goal point in group $\G_{\iterm}$), $y_{m-1}$ has never been a member of an expansion group:
\begin{align*}
\iterm &\leq \max\{i_{m-1}, \lceil c(y_{m-1}) / \delta\rceil\}\\
       &\leq \max\{\lceil S_{m-2} / \delta\rceil + m-1, \lceil c(y_{m-1}) / \delta\rceil\}.
\end{align*}
Then since the algorithm terminates at iteration $\iterm$,
\begin{align*}
c_\text{\gmt} \leq \iterm\delta &\leq \max\{S_{m-2} + \delta + (m-1)\delta, c(y_{m-1}) + \delta\} \\
    &\leq S_{m-1} + m \delta.
\end{align*}

\noindent Case 2b: $y_m$ must be connected to some parent when it is first considered and $y_{m-1}$ is a candidate, so
\[
c(y_m) \le c(y_{m-1}) + \|y_{m-1} - y_{m}\| \le S_m + (m-2)\delta.
\]
Depending on whether $y_{m-1}$ spends one or more steps in $\vopen$, either $i_m = i_{m-1} + 1$ or $i_m = \lceil c(y_{m-1}) / \delta \rceil + 1$. In the first subcase, 
$
i_m \leq \left\lceil \frac{S_{m-2}}{\delta}\right\rceil + m
$;
in the second subcase,
$
i_m \leq \lceil S_{m-1} / \delta \rceil + m - 1.
$

\noindent Case 2c: When $y_m$ is first considered for connection during iteration $i' < i_m \leq i_{m-1}$, there must be some $z \in \G_{i'}$ such that $\|y_m - z\| \leq r$. Then
\begin{align*}
c(y_m) & \leq c(z) + r \leq {i'}\delta + r \leq (i_{m-1} - 1)\delta + r\\
       & \leq S_{m-2} + (m-1)\delta + r.
\end{align*}
Recalling that the $y_m$, by construction, are spaced so that they satisfy the property $r < \|y_m - y_{m-2}\| \leq \|y_m - y_{m-1}\| + \|y_{m-1} - y_{m-2}\|$, we have
$
c(y_m) \leq S_m + (m-1)\delta.
$
Additionally, $i_m \leq i_{m-1} \leq \left\lceil \frac{S_{m-2}}{\delta}\right\rceil + m - 1$.

Thus the inductive step holds in all cases and \eqref{eqn:DPclaim} or \eqref{eqn:stepbound} holds for all $m$. In particular taking $m = M$ we have $c_\text{\gmt} \leq S_M + M\delta$ or $c_\text{\gmt} \leq c(y_M) \leq S_M + (M-1)\delta$. Noting that $S_M \geq Mr/2 = M\delta/(2\lambda)$, we have
\[
c_\text{\gmt} \leq (1 + 2\dr) \sum_{k=1}^{M} \|y_k - y_{k-1}\|
\]
as desired.
\end{proof}

Given these results, we are now in a position to prove asymptotic optimality within a suboptimality bound.

\begin{theorem}[\gmt Approximate Asymptotic Optimality]\label{thm:GMTao}
Assume a $\delta_\text{obs}$-robustly feasible\footnote{Briefly, we require that $\sigma^\ast$ is a limit of paths with bounded clearance from $\xobs$; this may be regarded as a minimum regularity assumption to guard against problems with passages of infinitesimal width that are not amenable to sampling-based motion planning.} path planning problem, as defined in \cite{SchmerlingJansonEtAl2015}, with optimal path $\sigma^\ast$ and cost $c^\ast$. Let $c_n$ denote the cost returned by \gmt with $n$ samples.
Then for any $\epsilon > 0$,
$$\lim_{n\to\infty} \text{P}[c_{n} > (1+\epsilon)(1+2\dr)c^\ast] = 0.$$
\end{theorem}
\begin{proof}
The proof of this theorem is conceptually identical to Theorem VI.2 in \cite{SchmerlingJansonEtAl2015}; with high probability as $n\to\infty$ there exists a sequence of waypoints (Theorem~\ref{thm:GMTpe}) tracing an obstacle-clear, near-optimal path of cost $\leq (1+\epsilon)c^\ast$ which \gmt recovers up to a factor $(1+2\dr)$ (Theorem~\ref{thm:GMTbound}).
\end{proof}

We remark that extending the results of this section to differentially-constrained system dynamics or deterministic sampling methods does not substantially change the argument in Theorem~\ref{thm:GMTao} (although in particular, with deterministic sampling the convergence in probability may be replaced with a standard limit). All that is required is an analogue of Theorem~\ref{thm:GMTpe}, a statement on the regularity of the dynamics that implies, with sufficient sample density, that the near-neighbor disk graph contains near-optimal paths. The only \gmt-specific analysis in Theorem~\ref{thm:GMTbound} is a graph search bound---independent of dynamics or sampling methodology. 

\subsection{Numerical Experiments: Suboptimality Introduced}\label{sec:numexp}

To complement the above theoretical bounds, in this subsection we examine the suboptimality incurred in practice through numerical experiments. 
While we will later describe implementation details and timing results in a few representative problems, this section's focus is solely on the amount of suboptimality resulting from the group expansion. 
Our figure of merit is thus only the percentage cost increase compared to \fmt, i.e., from the group expansion.

Table~\ref{table:lambda} lists results for two geometric planning problems over a variety of dimensions (2D to 10D), the first of which is shown in Fig.~\ref{fig:gmt}. 
This obstacle set was mapped to dimensions greater than two by expanding obstacles to fully fill the space.
Fig.~\ref{fig:gmt_3obs_wave}, in particular, shows the wave-like structure of the parallel expansion. 
The second planning problem listed in Table~\ref{table:lambda} is a maze environment requiring exploration in all dimensions.
For each planning problem, the same setup is run with varying $\dr$ (values of $0.2$, $0.5$, and $1.0$) and with sufficiently high sample counts to nearly converge to the optimal.
In each case, the suboptimality is significantly below the proven bound, and often below 5\%.
We observe the expected increase in cost error with increasing $\dr$, and an additional increase with increasing dimension.

\begin{figure}[h!tbp]
\centering
\begin{subfigure}[b]{0.24\textwidth}
    \includegraphics[width=\textwidth]{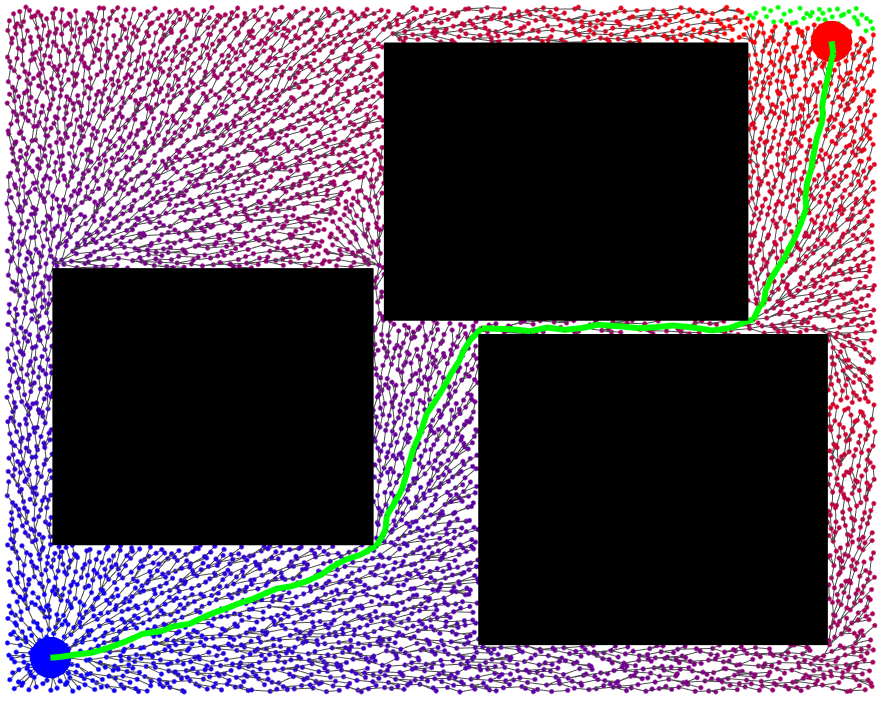}
    \caption{}\label{fig:gmt_3obs_tree}
\end{subfigure}
\begin{subfigure}[b]{0.24\textwidth}
    \includegraphics[width=\textwidth]{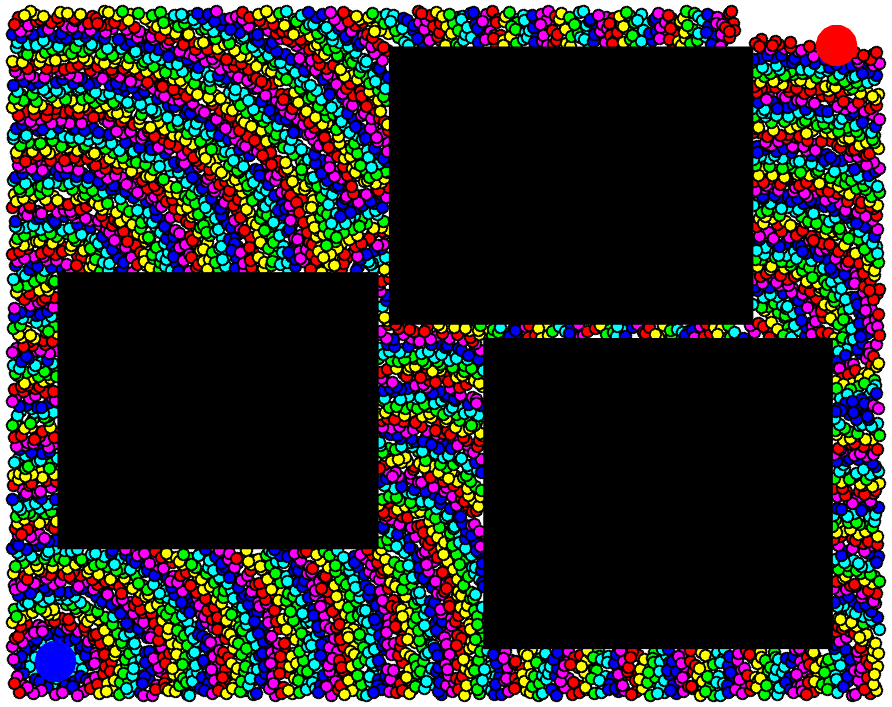}
    \caption{}\label{fig:gmt_3obs_wave}
\end{subfigure}
\caption{Expansion of the \gmt algorithm, where Fig.~\ref{fig:gmt_3obs_tree} shows the resulting tree (colored by cost) and Fig.~\ref{fig:gmt_3obs_wave} shows individual groups (denoted by color) expanded in parallel.
}
\label{fig:gmt}
\end{figure}

\begin{table}[h!tbp]
\begin{center}  
\begin{tabular}{c c| r r r}
  \multicolumn{2}{c}{} & \multicolumn{3}{c}{Cost Error ($c_{\text \gmt}/c_{\text \fmt}-1$)} \\ 
  \hline
  Obstacle & $d$ & $\dr = 0.2$ & $\dr = 0.5$ & $\dr = 1.0$ \\
  \hline
  Rectangles & 2D & 0.2\%  & 0.6\% & 1.8\% \\
  (Fig.~\ref{fig:gmt_3obs_tree}) & 3D & 0.1\% & 0.6\% & 3.4\% \\
   & 6D  & 0.4\% & 1.5\% & 2.1\% \\
   & 10D & 2.0\% & 14.8\% & 17.0\% \\
  \hline 
  Maze & 3D & 0.3\% & 1.5\% & 4.7\% \\ 
   & 5D  & 0.9\% & 4.9\% & 7.8\% \\
  \hline
  \end{tabular}
\end{center}
\caption{Suboptimality introduced by \gmt over \fmt for a range of dimensions $d$ and group cost threshold factors $\dr$. Results are averaged over 50 runs at $n = 5000$ samples. The variance was found to be small.}
\label{table:lambda}
\end{table}

\section{GPU Implementation}\label{sec:gpu}

We begin this section with a brief discussion of GPU architectures, as the ability of \gmt to exploit the computational capabilities of many-core GPUs is fundamental to this work and has driven much of the algorithm design and implementation.
We particularly focus here on the CUDA enabled GPUs used in this work. 
CUDA C functions running on GPUs are organized in a three level thread-hierarchy. 
At the lowest level, threads run in groups of 32 that execute one common instruction at a time, i.e., any divergence will cause branches to execute serially. 
A level above this, thread groups are combined into blocks, each of which can utilize a small, low-latency shared memory block and executes concurrently on the same multiprocessor, but is allowed to diverge without causing serial execution.
Finally, at the highest level, blocks are formed in grids to be dispatched to the device.
A more detailed discussion can be found in \cite{ios_NVIDIA:2016}.

We highlight here three properties of \gmt that, along with the sample-level parallelism, allow efficient application to GPU architectures. 
First, the use of lazy collision checking limits thread divergence at low levels by only attempting to connect new samples to the tree once per iteration.
Second, the design of \gmt is such that the sample set is partitioned into $\vunexplored$, $\vopen$, and $\vclosed$, with a sample always a member of one and only one set, allowing for little overlap of memory access and easy memory representation as Boolean masks. 
Our work accesses these sets with thread identifiers assigned via prefix sums, a strategy described in \cite{MerrillGarlandEtAl2012}. 
The use of this algorithmic primitive allows fast reorganization of sparse and uneven workloads into dense uniform ones.
Third, as the set of samples considered for expansion, $\G$, can be represented as a set of cost-thresholded buckets, there is no need for the use of serial data structures, e.g., min-heaps.

\section{Numerical Experiments}\label{sec:mp}

\subsection{Numerical Experiment Setup}\label{sec:exp}

As our goal is to show planning in changing, uncertain settings with dynamic systems, in this section we apply \gmt to the problem of planning under differential constraints with a 6D double integrator ($\ddot{x} = u$) and a Dubins airplane (Dubins car with altitude \cite{ChitsazLaValle2007}).
The double integrator planning problems consider a mixed time/quadratic control effort cost function, while the Dubins airplane problems consider an Euclidean distance cost function.
The algorithm was implemented in CUDA C (example code may be accessed at \url{github.com/StanfordASL/GMT}) and run on an NVIDIA GeForce GTX 980 GPU on a Unix system with a 3.0 GHz CPU.
We additionally provide comparison with an embeddable GPU, the NVIDIA Jetson TX1, to show these performance gains are similarly available for onboard computation.
Our implementation of \gmt samples the state space using the deterministic, low-dispersion Halton sequence, to achieve best performance, following the discussion in \cite{JansonIchterEtAl2015}. 
Sampling and computation of nearest neighbor connections (edge discretization and neighbor sets) were performed offline in a precomputation phase.

\subsection{Planning Under Differential Constraints}\label{sec:diff}

Through several motion planning problems, detailed in Fig.~\ref{fig:problems} and Table \ref{table:kinoMP}, we demonstrate \gmt achieves one to two orders of magnitude speed up with relatively small performance losses compared to an implementation of \fmt on a CPU \cite{JansonSchmerlingEtAl2015} and an implementation of \prm on a GPU \cite{KaramanFrazzoli2011}.
For each simulation, we pick a value for the connection radius $r$ appropriate for the dynamics \cite{SchmerlingJansonEtAl2015}\cite{SchmerlingJansonEtAl2015b} and set $\dr$ to 1, which we have found allows simple implementation, maximum parallelization, and performance losses on the order of 10\%. 
The obstacles in our simulation are represented by unions of axis-aligned bounding boxes, as commonly used for a broad phase collision checking phase \cite{Lavalle2006}.
This methodology can provide increasingly accurate representations of obstacle sets as more are used (e.g., as in Fig.~\ref{fig:indoor} or with octree-based representations as in Fig.~\ref{fig:quad_caves}).

The first problem (Fig.~\ref{fig:indoor}) was built from point cloud data collected in \cite{ArmeniSenerEtAl2016} for an indoor office environment, with individual environmental elements bounded by boxes. 
The second planning problem (Fig.~\ref{fig:quad_caves}) represents a cave system consisting of two maze-like levels connected by three passageways.
Finally, our third planning problem (Fig.~\ref{fig:forest}) represents a forest environment.
To show the results extend to systems with nonlinear dynamics (and nonholonomic planning), this last simulation uses Dubins airplane dynamics rather than double integrator dynamics. 
Our Dubins airplane consists of a planar Dubins car augmented with a single integrator in the third dimension, with bounded control on turning rate, unbounded altitude control, and a Euclidean distance cost function \cite{ChitsazLaValle2007}.

As shown in Table~\ref{table:kinoMP}, in all problems a solution trajectory is found by \gmt in \texttildelow10~ms, a speed increase of two orders of magnitude over CPU \fmt and one order of magnitude over GPU \prm. 
The algorithm also performs well on the embedded platform, only slowing by a factor of two, compared to \prm, which slows down by approximately a factor of five. 
This demonstrates \gmt's lightweight approach of building a single tree is particularly amenable to onboard computation. 
We also note that the cost increase incurred is less than 12\% for all cases despite the high group cost threshold factor.

\begin{figure}[htbp]
\centering
\begin{subfigure}[b]{0.238\textwidth}
    \includegraphics[width=\textwidth]{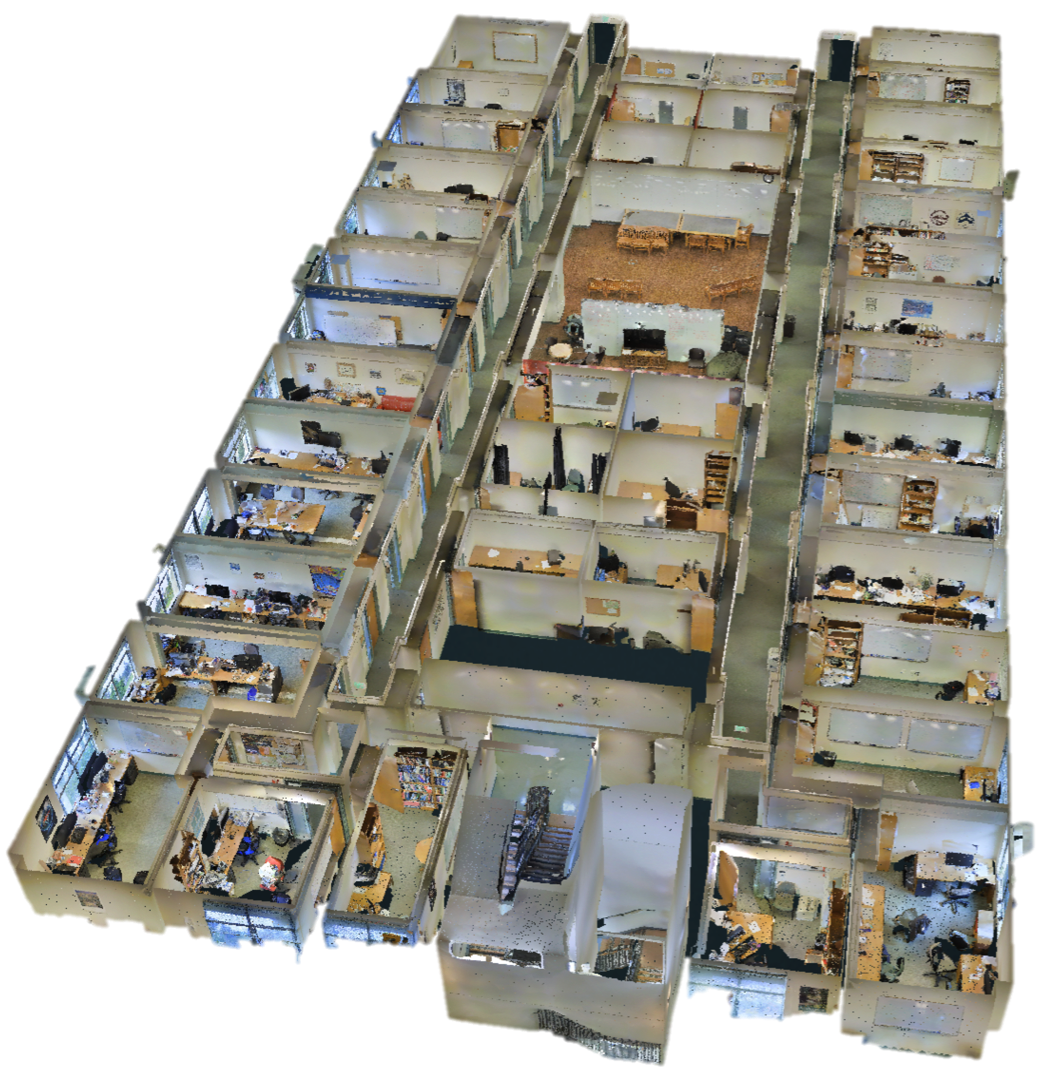}
    \caption{}\label{fig:indoorPointCloud}
\end{subfigure}
\begin{subfigure}[b]{0.244\textwidth}
    \includegraphics[width=\textwidth]{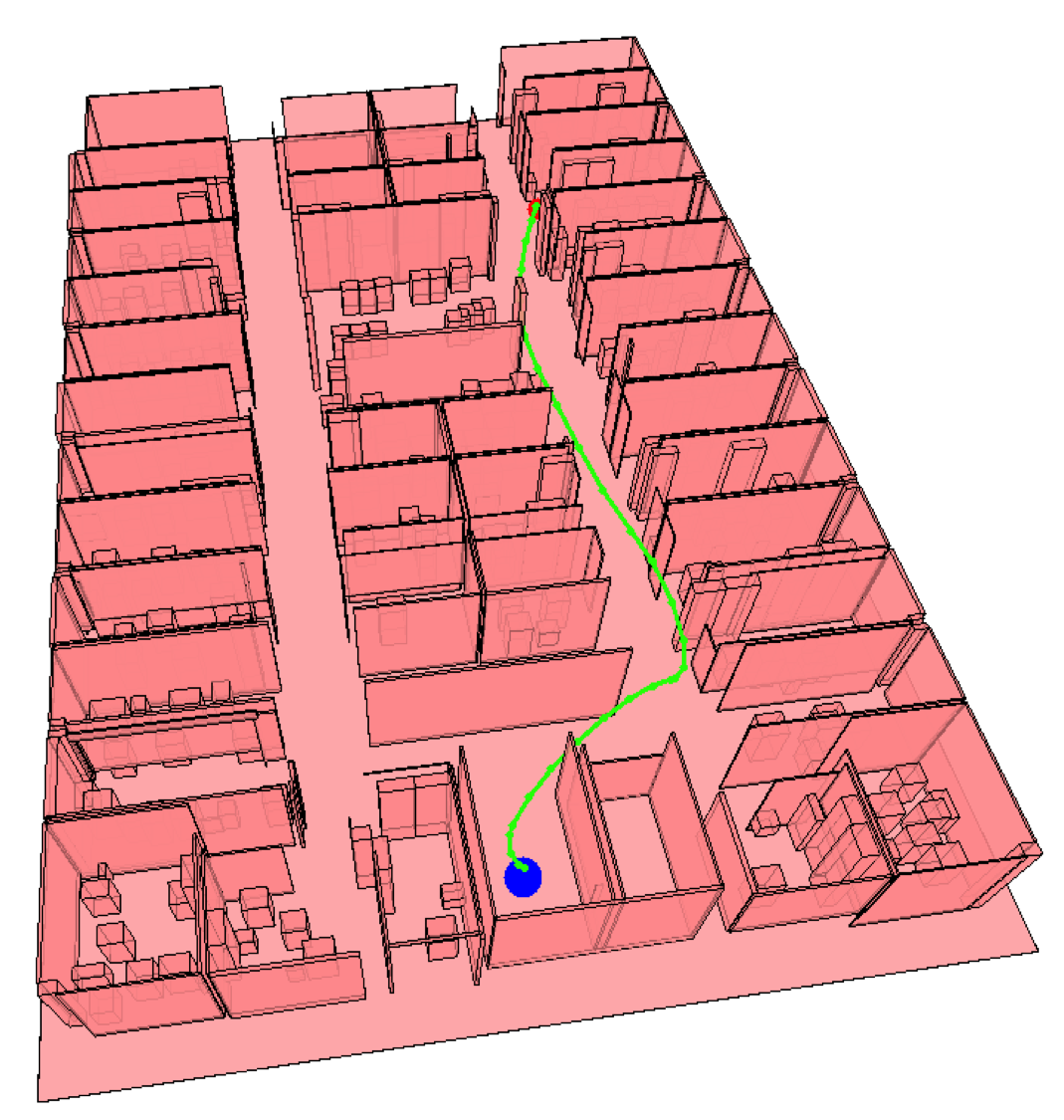}
    \caption{}\label{fig:indoor}
\end{subfigure}
\begin{subfigure}[b]{0.241\textwidth}
    \includegraphics[width=\textwidth]{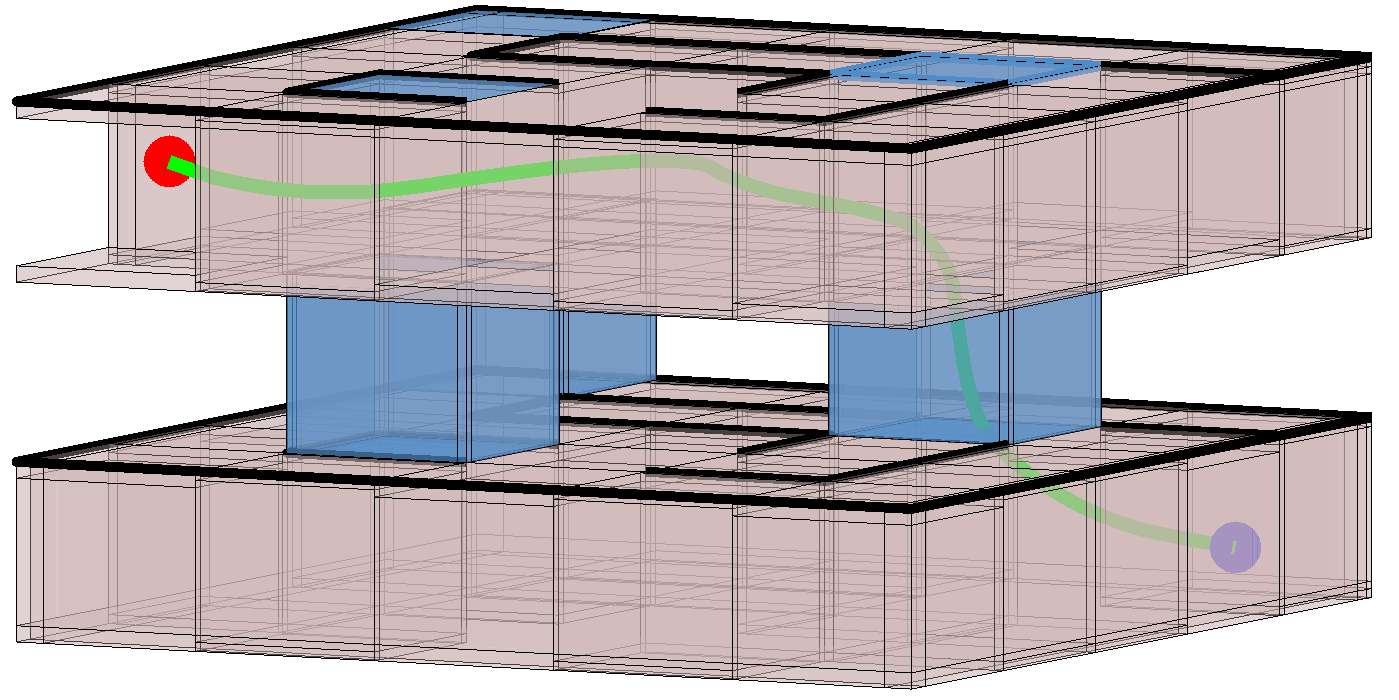}
    \caption{}\label{fig:quad_caves}
\end{subfigure}
\begin{subfigure}[b]{0.241\textwidth}
    \includegraphics[width=\textwidth]{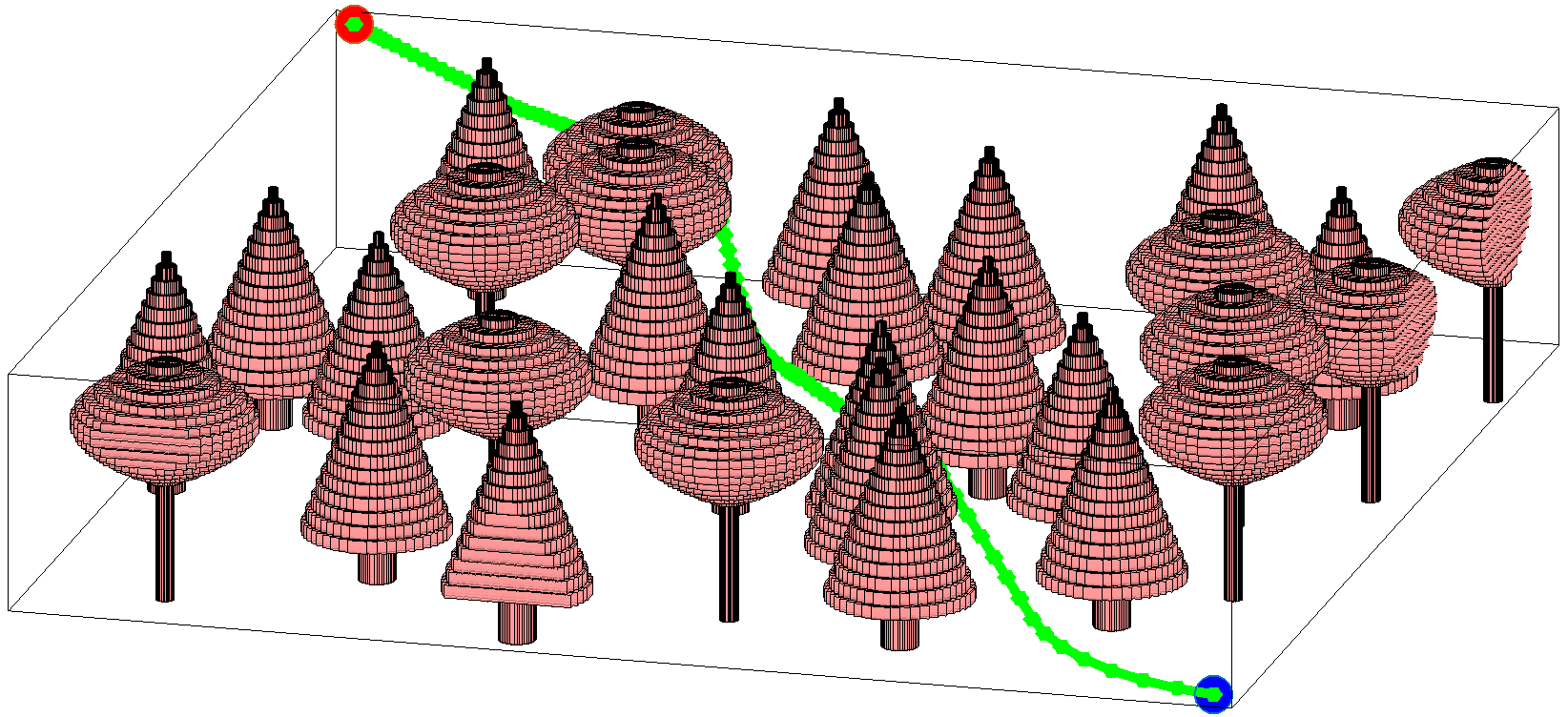}
    \caption{}\label{fig:forest}
\end{subfigure}
\caption{
The solution trajectory connecting $\xinit$ (blue) to $\xgoal$ (red) returned by \gmt is shown in green for all figures.
(\ref{fig:indoorPointCloud}-\ref{fig:indoor}) The indoor environment was generated with point cloud data from \cite{ArmeniSenerEtAl2016} with individual elements bounded by boxes. 
(\ref{fig:quad_caves}) The cave environment consists of two maze-like levels, with wall outlines shown in black, connected by three passageways, shown in blue. 
(\ref{fig:forest}) The forest environment consists of many varying tree obstacles.
}
\label{fig:problems}
\end{figure}

Table~\ref{table:scaling} further demonstrates the algorithm's scaling with sample and obstacle counts.
The small increases in computation time with increasing sample count are a result of the GPU not being fully utilized at every iteration with lower sample counts, i.e., the group size may be too small to use every GPU core.
The obstacle scaling too shows only slight increases in computation time with increased obstacle resolution, approximately doubling for each order of magnitude increase, however, if obstacles and complexity of the space becomes a significant bottleneck, \gmt is amenable to space partitioning structures, e.g., k-d trees, or parallelization at the obstacle level \cite{BialkowskiKaramanEtAl2011}.

\begin{table}[htbp]
\begin{center}  
\begin{tabular}{c c c c c}
  \multicolumn{5}{c}{Double Integrator, Fig.~\ref{fig:indoor}} \\
  \hline
  Algorithm & Device & $c_{\text{alg}}/c_{\text{\gmt}}$ & Time (ms) &  $t_{\text{alg}}/t_{\text{\gmt}}$ \\
  \hline
  \fmt & CPU & 0.91 & 1291 & 129.1 \\
  \prm & Embd. GPU & 0.88 & 735 & 73.5 \\
  \prm & GPU & 0.88 & 158 & 15.8 \\
  \gmt & Embd. GPU & 1 & 27 & 2.7 \\
  \gmt & GPU & 1 & 10 & 1 \\
  \hline
  \multicolumn{5}{c}{\rule{0pt}{2.5ex}    Double Integrator, Fig.~\ref{fig:quad_caves}} \\
  \hline
  Algorithm & Device & $c_{\text{alg}}/c_{\text{\gmt}}$ & Time (ms) &  $t_{\text{alg}}/t_{\text{\gmt}}$ \\
  \hline
  \fmt & CPU & 0.91 & 1490 & 99.3 \\
  \prm & Embd. GPU & 0.90 & 517 & 34.5 \\
  \prm & GPU & 0.90 & 140 & 9.3 \\
  \gmt & Embd. GPU & 1 & 31 & 2.0 \\
  \gmt & GPU & 1 & 15 & 1.0 \\
  \hline
  \multicolumn{5}{c}{\rule{0pt}{2.5ex}    Dubins Airplane, Fig.~\ref{fig:forest}} \\
  \hline
  Algorithm & Device & $c_{\text{alg}}/c_{\text{\gmt}}$ & Time (ms) &  $t_{\text{alg}}/t_{\text{\gmt}}$ \\
  \hline
  \fmt & CPU & 0.95 & 1312 & 218.7 \\
  \prm & Embd. GPU & 0.95 & 945 & 157.5 \\
  \prm & GPU & 0.95 & 96 & 16.0 \\
  \gmt & Embd. GPU & 1 & 11 & 1.8 \\
  \gmt & GPU & 1 & 6 & 1 \\
  \hline
  \end{tabular}
\end{center}
\caption{Results for algorithms run with 5000 samples in the environments in Fig.~\ref{fig:problems}. GPU refers to the GTX 980, while \ Embd. GPU refers to an embeddable Jetson TX1 GPU.}
\label{table:kinoMP}
\end{table}

\begin{table}[htbp]
\begin{center}  
\begin{tabular}{c c c c c}
  \hline
  Fig. & Sample Count ($n$) & Obstacle Count & Time (ms) &  Cost \\
  \hline
  	\ref{fig:quad_caves} & 1k & 300 & 6 & 552 \\
	& 2k & 300 & 8 & 379 \\
	& 5k & 300 & 15 & 290 \\
	& 10k & 300 & 21 & 233 \\
	\hline
	\ref{fig:forest} & 5k & 150 & 7 & - \\ 
	& 5k & 500 & 14 & -  \\ 
	& 5k & 1500 & 18 & -  \\ 
	& 5k & 5000 & 26 & -  \\ 
  \hline
  \end{tabular}
\end{center}
\caption{Results for \gmt scaling with sample and obstacle counts, in terms of cost and computation time. Note the obstacle count represents varying the resolution of the obstacle representations, not varying the obstacle location or class.}
\label{table:scaling}
\end{table}

\subsection{Planning in the Loop}\label{sec:loop}

The numerical experiments in Section~\ref{sec:diff} have shown that it is possible to plan at rates amenable to implementation within control loops by allowing some performance loss in exchange for parallelism.
We now demonstrate that this strategy is beneficial through numerical experiments for a system operating in a dynamic environment with random state disturbances.
The setup mimics a collapsing cave system (Fig.~\ref{fig:quad_caves}), which a quadrotor modeled as a double integrator must escape.
A successful escape requires high performance actions to minimize time spent in the degrading cave as well as actions that account for state disturbances and variations in the environment. 
The collapse is modeled as randomly placed box obstacles added to the environment, with the rates representing the number of obstacles added each second. 
Fig.~\ref{fig:collapsing} shows the success rate over 50 runs of a quadrotor using a waypoint tracking controller, which tracks trajectories generated with \fmt and \gmt (replanning as quickly as possible). 
The results for \fmt show that, as expected, success rate decreases quickly with increased noise and environmental degradation. 
Replanning with \gmt, however, shows little variation in failure rate with increased noise level.
The results further show significantly higher success rates for replanning with \gmt than replanning with \fmt.
Note that these planning problems may not be possible to solve for every instance, as the collapses can happen anywhere within the environment (and very quickly), potentially trapping the quadrotor. 
Experiment videos are available at \url{https://goo.gl/67RSsp}. 
\begin{figure}[htbp]
\centering
    \includegraphics[width=0.48\textwidth]{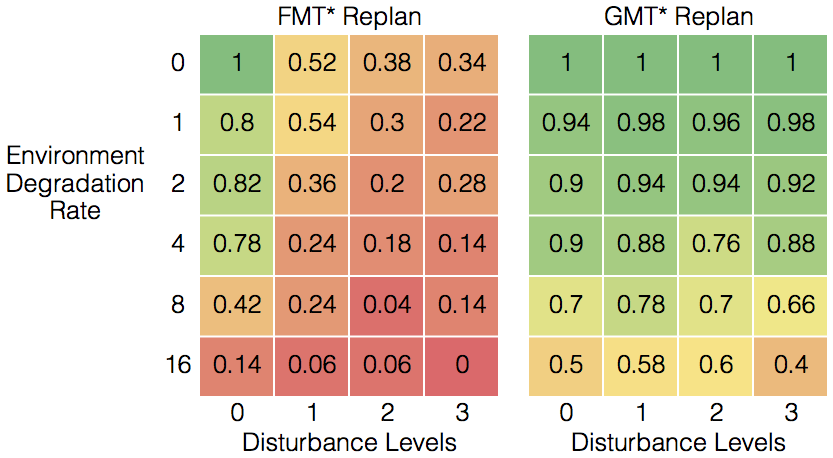}
\caption{Success rate of a quadrotor replanning with \fmt or \gmt in the Fig.~\ref{fig:quad_caves} environment with varying levels of state disturbances and cave collapse rates (simulated with obstacles randomly placed in the environment).}
\label{fig:collapsing}
\end{figure}

\section{Conclusion}\label{sec:concl}

We have introduced and analyzed a novel planning algorithm, the Group Marching Tree algorithm (\gmt), that trades off parallelism for optimality in order to leverage GPU hardware.
The computational speed of \gmt allows us to approach the problem of planning in real-world settings---particularly focusing on the uncertain, dynamic environments that naturally arise from active robot sensing and the uncertain, disturbed motion of systems in the field---by replanning at rates commensurate with the control loop frequency.
Simulation results show planning times on the order of 10~ms (for a 6D double integrator and Dubins airplane) and demonstrate the efficacy of planning at these rates in difficult environments.

This paper leaves several important research avenues open. Foremost, we plan to validate this approach experimentally on a platform with state and environmental sensing. 
We further plan to provide a more detailed theoretical analysis of \gmt, such as providing time and space complexity analysis and potentially proving tighter suboptimality bounds.
We additionally plan to explore extensions to other planning paradigms, which the computational speed of \gmt may enable. 
To merge planning and game playing, we plan to use \gmt as a default simulation policy when many actions must be considered, such as in algorithms like Monte Carlo tree search. 
We also plan to show that \gmt may be used to construct policies in decision making frameworks through fast approximation of the cost-to-go. 
Finally, we plan to demonstrate extensions to planning with a probabilistic state belief by utilizing a backwards search in cost-to-go space.
In this way we can define actions over regions of the state space, with the same computation times shown above, and select actions from criteria such as best worst-case or maximum expected performance.

\bibliographystyle{IEEEtran-short}
{
\renewcommand{\baselinestretch}{0.92}
\bibliography{../../../../bib/main,../../../../bib/ASL_papers}
}

\end{document}